\documentclass{article}

\usepackage{ijcai18}

\usepackage{times}
\usepackage{soul}
\usepackage{url}
\usepackage[hidelinks]{hyperref}
\usepackage[utf8]{inputenc}
\usepackage[small]{caption}

\usepackage{amsmath}
\usepackage{amssymb}
\usepackage{amsthm}
\usepackage{enumerate}

\usepackage{xspace}
\usepackage{tikz}

\newcommand{\Omit}[1]{}
\newcommand{\denselist}{\itemsep -1pt\partopsep 0pt}
\newcommand{\tup}[1]{\langle #1 \rangle}
\newcommand{\pair}[1]{\langle #1 \rangle}
\newcommand{\citeay}[1]{\citeauthor{#1} [\citeyear{#1}]}

\newtheorem{definition}{Definition}

\newtheorem{theorem}[definition]{Theorem}

\newcommand{\alert}[1]{\textcolor{red}{\bf #1}}

\newcommand{\Q}{\mathcal{Q}}

\newcommand{\Eff}{{\mathit{Eff}}}

\newcommand{\abst}[2]{\tup{#1;#2}}
\newcommand{\Rule}[2]{\ensuremath{#1 \Rightarrow #2}}

\newcommand{\pplus}{\hspace{-.05em}\raisebox{.15ex}{\footnotesize$\uparrow$}}
\newcommand{\mminus}{\hspace{-.05em}\raisebox{.15ex}{\footnotesize$\downarrow$}}

\newcommand{\Stack}{\text{Stack}}
\newcommand{\Unstack}{\text{Unstack}}
\newcommand{\Pickup}{\text{Pickup}}
\newcommand{\Putdown}{\text{Putdown}}
\newcommand{\Move}{\text{Move}}

\newcommand{\Example}{\medskip\noindent\textbf{Example.}\xspace}

\title{Features, Projections, and Representation Change for Generalized Planning}

\author{
Blai Bonet$^1$ {\normalfont and}
Hector Geffner$^2$
\\
$^1$ Universidad Sim\'on Bol\'{\i}var, Caracas, Venezuela\\
$^2$ ICREA \& Universitat Pompeu Fabra, Barcelona, Spain\\
bonet@usb.ve,
hector.geffner@upf.edu
}

\begin{document}

\maketitle

\begin{abstract}
Generalized planning is concerned with the characterization
and computation of plans that solve many instances at once. 
In the standard formulation, a generalized plan is a mapping
from feature or observation histories into actions, assuming
that the instances share a common pool of features and actions.
This assumption, however, excludes the standard relational
planning domains where actions and objects change across
instances. 
In this work, we extend the standard formulation of generalized planning
to such domains. This is achieved by projecting the actions over
the features, resulting in a common set of \emph{abstract actions}
which can be tested for soundness and completeness, and which can
be used for generating general policies such as ``if the gripper
is empty, pick the clear block above $x$ and place it on the table''
that achieve the goal $clear(x)$ in any Blocksworld instance.
In this policy, ``pick the clear block above $x$'' is an abstract
action that may represent the action $\Unstack(a,b)$ in one
situation and the action $\Unstack(b,c)$ in another.
Transformations are also introduced for computing such policies
by means of fully observable non-deterministic (FOND) planners.
The value of generalized representations for \emph{learning}
general policies is also discussed.
\end{abstract}

\section{Introduction}

Generalized planning is concerned with the characterization and computation
of plans that solve many instances at once 
\cite{srivastava08learning,bonet09automatic,srivastava:generalized,hu:generalized,BelleL16,anders:generalized}.
For example, the policy ``if left  of the  target, move right'' and ``if 
 right of the  target, move left'', solves the problem of getting to the
target in a $1\times n$ environment, regardless of the agent and
target positions  or  the value of $n$.

The standard, semantic, formulation of generalized planning due to \citeay{hu:generalized}
assumes that all the problems
in the class share a common set of features and actions. Often, however, this
assumption is false. For example, in the Blocksworld, the policy ``if the gripper
is empty, pick the clear block above $x$ and place it on the table'' eventually
achieves the goal $clear(x)$ in \emph{any} instance, yet the set of such instances
do not have (ground) actions in common. Indeed, the expression ``pick up the clear
block above $x$'' may mean ``pick block $a$ from $b$'' in one case, and  ``pick
block $c$ from $d$'' in another. A similar situation arises in most of the standard
relational planning domains.
Instances  share the same action \emph{schemas} but policies cannot
map features into action schemas; they must select concrete, ground, actions.

In this work, we show how to extend the formulation of generalized planning to
relational domains where the set of actions and objects depend on the instance.
This is done by \emph{projecting the actions over a  common set of features},
resulting in a common set of \emph{general actions} that can be used in
generalized plans.
We also address the computation of the resulting general  policies 
by means of  transformations and \emph{fully observable non-deterministic (FOND)}
planners, and discuss the relevance to  work on \emph{learning general policies}.
%

\Omit{
The differences with generalized planning are
that the generalized learning approaches are inductive and provide no guarantees on the conditions
under which the learned policies will work. Generalized planning approaches are more demanding on
the inputs that  they require (models and features) and are  more challenging computationally,
but are better suited for providing the elements for understanding and ensuring generalization
across different problem instances.
}

Generalized planning has also been formulated as a problem in first-order logic
with solutions associated with programs with loops, where actions schemas do not
need to be instantiated \cite{srivastava:generalized}. The resulting formulation,
however, is complex and cannot benefit from existing propositional planners.
First-order decision theoretic planning can also be used to generate general
plans but these are only effective over finite horizons
\cite{boutilier2001symbolic,wang2008first,van2012solving}.

The paper is organized as follows. We review the definition of generalized
planning and introduce abstract actions and projections. We then consider the computation of
generalized policies, look at  examples,  and discuss related work and challenges.

\Omit{
Connections to generalized policy size, 
of width in classical planning \cite{nir:ecai2012}, and discuss the value of
generalized representations for planning and learning and their relations to
agent or deictic representations \cite{ballard:deictic,agre:deictic,kaelibling:deictic,barto:agent-centered}.
\alert{Check this!}
}

\section{Generalized Planning}

The planning \emph{instances} $P$ that we consider are classical planning problems
expressed in some compact language as a tuple $P=\tup{V,I,A,G}$ where $V$ is a
set of state variables that can take a finite set of values (boolean or not), $I$
is a set of atoms over $V$ defining an initial state $s_0$, $G$ is a set of atoms
or literals over $V$ describing the goal states, and $A$ is a set of actions $a$
with their preconditions and effects that define the set $A(s)$ of actions applicable
in any state $s$, and the successor state function $f(a,s)$, $a \in A(s)$.
A state is a valuation over $V$ and a solution to $P$ is an applicable action
sequence $\pi=a_0,\ldots,a_n$ that generates  a state sequence $s_0,s_1,\ldots,s_{n}$
where $s_n$ is a goal state (makes $G$ true). In this sequence, $a_i \in A(s_i)$ and $s_{i+1}=f(a_i,s_i)$ for $i=0, \ldots, n-1$.
A state $s$ is reachable in $P$ if $s=s_n$ for one such sequence.

A general planning problem $\Q$ is a collection of instances $P$ assumed to share
a common set of actions and a common set of features or observations \cite{hu:generalized}.
A \emph{boolean feature} $p$ for a class $\Q$ of problems represents a function $\phi_p$
that takes an instance $P$ from $\Q$ and a reachable state $s$ in  $P$,\footnote{
 Non-reachable states are often not meaningful,  like a state when one block is 
 on top of itself.
}
and results in a boolean value denoted as $\phi_p(s)$, with
the reference to the problem $P$ omitted. 
Features are sometimes associated with observations, and in such cases $\phi_p$ is
a sensing function.
For example, if $\Q$ is the class of all block instances with goal $on(a,b)$ for two
blocks $a$ and $b$, $above(a,b)$ would be a feature that is true when  $a$ is
above $b$. If $F$ denotes the set of all boolean features, $\phi_F(s)$ denotes
the values of all features in the state $s$.

A \emph{general policy} or \emph{plan} for a class of problems $\Q$ sharing a common set of boolean features
and actions is a partial function  $\pi$ that maps feature histories into actions.
A policy $\pi$ solves the generalized problem $\Q$ if it solves each instance $P$ in $\Q$.
A policy $\pi$ solves  $P$ if the state trajectory $s_0,\ldots,s_{n}$
induced by $\pi$ on $P$ is goal reaching. In such  trajectory, 
$a_i=\pi(\phi_F(s_0),\ldots,\phi_F(s_i))$,  
$s_{i+1}=f(a_i,s_i)$ for $i<n$, and $s_{n}$ is the first state in the sequence where
the goal of $P$ is true, the policy $\pi$ is not defined, or $\pi$ returns an action $a_{n}$ that is not applicable at $s_n$.
In the following, for simplicity, we only consider \emph{memoryless policies}
that map single feature valuations into actions and thus  $a_i=\pi(\phi_F(s_i))$.

\Omit{
As an illustration, if $\Q$ represents the class of problems $P$ where an agent has
to reach a target by moving  up, down, right, or left in a $n \times n$ grid,
and $u$, $d$, $r$, $l$ are four boolean features representing whether the target
is above, below,  right, or  left  of the agent, the (memoryless) policy
encoded by the rules ``if $u$, move-$u$'', ``if $d$, move-$d$'', ``if not $u$, not $d$,
and $r$, move-$r$'', and ``if not $u$, not $d$, and $l$, move-$l$''
solves $P$ for any initial location for the agent and target, and any value of $n$.
}

\subsection{Numerical Features}

For extending the formulation to domains where actions take arguments that
vary from instance to instance, we need numerical features.
A \emph{numerical feature} $n$ for a generalized problem $\Q$ represents a
function $\phi_n$ that takes an instance $P$ and a state $s$, 
and results in a non-negative integer value denoted as $\phi_n(s)$.
For the problem $\Q$ representing the Blocksworld instances with goal
$on(a,b)$, a numerical feature $n(a)$ can represent the number of blocks above
$a$. The set of features $F$ becomes thus a pair $F=\tup{B,N}$ of boolean and numerical
features $B$ and $N$. A \emph{valuation} over $F$ is an assignment of truth values
to the boolean features $p \in B$ and non-negative integer values to the numerical
features $n \in N$. A   \emph{boolean valuation} over $F$, on the other hand,
is an assignment of truth values to the boolean features $p \in B$ and to the
atoms $n=0$ for $n \in N$. The feature valuation  for  a state $s$ is denoted as $\phi_F(s)$,
while the boolean  valuation as $\phi_B(s)$. The number of feature valuations is infinite
but  the number of boolean valuations is $2^{|F|}$. Policies for a generalized problem $\Q$
over a set of features $F=\tup{B,N}$ are defined as:

\begin{definition}[Generalized Planning]
  \label{def:gp:1}
  A policy for a generalized problem $\Q$ over the features $F=\tup{B,N}$ is a
  partial  mapping $\pi$ from \emph{boolean valuations} over $F$ into the  common  actions in
  $\Q$. The policy $\pi$ solves $\Q$ if $\pi$ solves each instance $P$ in $\Q$; i.e.,
  if the trajectory $s_0,\ldots,s_{n}$ induced by the policy $\pi$ in $P$ is   goal reaching.
\end{definition}

As before, the trajectory $s_0,\ldots,s_{n}$ induced by $\pi$ is such that
$a_i=\pi(\phi_B(s_i))$,  
$s_{i+1}=f(a_i,s_i)$ for $i<n$,
and $s_{n}$ is the first state in the sequence where $\pi(\phi_B(s_{n})) = \bot$,  $a_{n} \not\in A(s_{n})$,
or the goal of $P$ is true.

\subsection{Parametric Goals and Features}

We are often interested in finding a policy for all instances of a domain whose goal has
a special form. For example, we may be interested in a policy for the class $\Q$ of all
Blocksworld instances with goal of the form $on(x,y)$ where $x$ and $y$ 
are  \emph{any} blocks. Two instances $P$ and $P'$ with goals $on(a,b)$
and $on(c,d)$ respectively, are   part of $\Q$  with  values for
$x$ and $y$  being $a$ and $b$ in $P$, and   $c$ and $d$ in $P'$.
General features using such parameters can  be used, with the value
of the parameters in an instance $P$ being determined by matching the goal of
$P$ with the generic goal.  A parametric feature like $n(x)$ that tracks the
number of blocks above $x$ thus  represents  the feature $n(a)$ in $P$
and the feature $n(c)$ in $P'$. 

%

\section{Abstract Actions}

A generalized planning problem like ``all Blocksworld instances with goal of the form $on(x,y)$''
admits simple and general plans yet such plans cannot be accounted for in the current framework.
The reason is that such instances share no common pool of actions. This is indeed the situation
in relational domains where objects and actions change across instances. 
For dealing with such domains, we define the notion of \emph{abstract actions}
that capture the effect of actions on the common set of features. 

The abstract actions for a generalized problem $\Q$ are defined as pairs $\bar{a}=\abst{Pre}{\Eff}$
where $Pre$ and $\Eff$ are the action precondition and effects expressed in terms of the set of
features $F=\tup{B,N}$. 
The syntax is simple: preconditions may include atoms $p$ and $n=0$ or their negations,  for $p \in B$
and $n \in N$, and effects may be boolean, $p$ or $\neg p$ for $p\in B$, or numerical increments
and decrements expressed as $n\pplus$ and  $n\mminus$ respectively for $n \in N$.
In the language of abstract actions, features represent state variables (fluents),
not functions, with $n$ representing a non-negative integer, and $n\pplus$ and  $n\mminus$
representing updates $n := n+\Delta$ and $n := n-\Delta'$ for \emph{random  positive integers}  $\Delta$
and $\Delta'$ that are not allowed to  make $n$ negative.  The numerical updates are thus \emph{non-deterministic}.
The negation of atom $n=0$ is expressed as $n>0$ and it must be a precondition of any abstract action
that decrease $n$ for ensuring that $n$ remains non-negative.

\Omit{
\alert{*** Two things here. First, minor thing, ``random'' increments/decrements may be confusing to some people.
Second, this one is relevant and important. We are dealing with unspecified positive, finite but \emph{unbounded},
increments.
The transition function $\bar F$ for the feature projection $Q_F$ (below) will have $|\bar F(\bar a,\bar s)|=\infty$ when $\bar a$ is applicable at $\bar s$ and has effect $n\pplus$.
Apparently (and interestingly) this is not causing any problem but it is something relevant as it is non-standard.  ***}
}

\begin{definition}
  \label{def:action:abstract}
  An \emph{abstract action} $\bar{a}$ over the features $F=\tup{B,N}$
  is a pair $\abst{Pre}{\Eff}$ such that 1)~each precondition in $Pre$
  is a literal $p$ or $\neg p$ for $p \in B$, or a literal $n=0$ or $n > 0$
  for $n \in N$, 2)~each effect in $\Eff$ is a boolean literal over $B$,
  or a numerical update $n\pplus$ or  $n\mminus$ for $n\in N$, and 3)~$n>0$
  is in $Pre$ if $n\mminus$ is in $\Eff$.
\end{definition}

We want abstract actions that represent the diverse set of concrete actions
in the different instances $P$ of the  generalized problem $\Q$. 
The number of concrete actions across all instances in $\Q$ is often
infinite  but the number of abstract actions is  bounded by $3^{2|F|}$.

Abstract actions operate over \emph{abstract states} $\bar{s}$ that are {valuations}
over the feature variables. An abstract action $\bar{a}=\abst{Pre}{\Eff}$  \emph{represents
a concrete  action $a$ over a state $s$ in instance $P$} if the two actions are
applicable in $s$, and they \emph{affect the values of the features in a similar way.}
Let us say that $Pre$ is true in $s$ if $Pre$ is true in the valuation $\phi_B(s)$.
Then:

\begin{definition}
  \label{def:action:represents}
  The abstract action $\bar{a} = \abst{Pre}{\Eff}$ over a set of features $F$
  \textbf{represents the action $a$ on the state $s$ of instance $P$} in $\Q$ iff
  1)~the preconditions of $a$ and $\bar{a}$ both hold in $s$, and 
  2)~the effects of $a$ and $\bar{a}$ over $F$ are similar:
  \begin{enumerate}[$a)$]
    \item for each boolean feature $p$ in $B$, if $p$ changes from true to false
      (resp.\ from false to true) in transition $s \leadsto f(a,s)$ then
      $\neg p\in\Eff$ (resp.\ $p\in\Eff$),
    \item for each boolean feature $p$ in $B$, if $p$ (resp.\ $\neg p$) is in $\Eff$,
      then $p$ is true (resp.\ false) in $f(a,s)$, and
    \item for each numerical feature $n$ in $N$, $n\mminus$ (resp.\ $n\pplus$)
      in $\Eff$ iff $\phi_n(f(a,s)) < \phi_n(s)$ (resp.\ $\phi_n(f(a,s)) > \phi_n(s)$).
  \end{enumerate}
\end{definition}


\Example
Let $\Q_{clear}$ be the set of Blocksworld instances $P$ with goal of the form $clear(x)$
and initial situation where the arm is empty,\footnote{Throughout, Blocksworld refers
  to the encoding with action schemas $\Stack(x,y)$, $\Unstack(x,y)$, $\Pickup(x)$, and $\Putdown(x)$.}
and let $F=\{H,n(x)\}$ be the set of features where $H$ holds iff the arm is holding some block, 
and $n(x)$ counts the number of blocks above $x$. The abstract action
\begin{equation}
  \bar{a}=\abst{\neg H,n(x)>0}{H,n(x)\mminus}
  \label{eq:ex1}
\end{equation}
represents any action that picks up a block from above $x$, as it  makes $H$
true and decreases the number of blocks above $x$. If $P$ is an instance with goal $clear(a)$ and $s$ is a state
where $on(b,a)$, $on(c,b)$, and $clear(c)$ are all true, $\bar{a}$ represents the action
$\Unstack(c,b)$ as both actions,  the abstract and the concrete, are applicable in $s$,
make $H$ true, and decrease $n(x)$. Likewise, 
\begin{equation}
\bar{a}'=\abst{H}{\neg H}
  \label{eq:ex2}
\end{equation}
represents any action that places the block being held anywhere but above $x$ (as it does not affect 
$n(x)$). In  the state $s'$ that results from  the state $s$ and the action $\Unstack(c,b)$,
$\bar{a}'$  represents  the action $\Putdown(c)$, and also $\Stack(c,d)$ if
$d$ is a block in $P$ that is clear in both $s$ and $s'$. \qed

\medskip\noindent
We say that an abstract action $\bar{a}$ is \emph{sound} when it represents \emph{some}
concrete action in each (reachable) state $s$ of each instance $P$ of $\Q$
where $\bar{a}$ is applicable:

\begin{definition}
  \label{def:action:sound}
  An abstract action $\bar{a}=\abst{Pre}{\Eff}$ is \emph{sound} 
  in the problem $\Q$ over the  features $F$  iff for each instance $P$ in $\Q$ and each reachable state
  $s$ in $P$ where $Pre$ holds, $\bar{a}$ represents one or more actions $a$ from $P$ on the state $s$. 
\end{definition}

\noindent 
The abstract action \eqref{eq:ex1} is sound in $\Q_{clear}$ as in any reachable state of any instance  where
the arm is empty and there are blocks above $x$, there is  an unstack action that makes the feature $H$
true and decreases the number of blocks above $x$. The abstract action \eqref{eq:ex2}  is sound too.
The two actions, however, do not provide a  \emph{complete} representation:

\begin{definition}
  \label{def:action:complete}
  A set $A$ of abstract actions over the set of features $F$ is \emph{complete} for $\Q$
  if for any instance $P$ in $\Q$, any \emph{reachable} state $s$ of $P$, and any action 
  $a$ that is applicable at $s$, there is an abstract action $\bar{a}$ in $A$ that
  represents $a$ on $s$.  
\end{definition}

The set made of the two actions above is not complete as they cannot represent concrete
actions that, for example, pick up the target block $x$ or a block that is not above $x$.

\Example
A sound and complete set of abstract actions $A_{F'}$  for $\Q_{clear}$
can be obtained with the  features $F'=\{X,H,Z,n(x),m(x)\}$
where $H$, $X$, and $Z$ represent that block $x$ is being held,  that
some other block is being held, and that there is a block below $x$
respectively. The new counter   $m(x)$  tracks the number of blocks that are not  in the same tower as $x$
or being held. The abstract actions in $A_{F'}$, with  names for  making  their meaning explicit, are:
\begin{enumerate}[--]
\item $\text{Pick-$x$-some-below}\!=\!\!\abst{\neg H,\!\neg X,\!n(x)\!\!=\!\!0,\!Z}{X,\!\neg Z,\!m(x)\pplus}$,
\item $\text{Pick-$x$-none-below} = \abst{\neg H,\neg X,n(x)=0,\neg Z}{X}$,
\item $\text{Pick-above-$x$} = \abst{\neg H,\neg X, n(x)>0}{H,n(x)\mminus}$,
\item $\text{Pick-other} = \abst{\neg H,\neg X,m(x) > 0}{H,m(x)\mminus}$,
\item $\text{Put-$x$-on-table} = \abst{X}{\neg X}$,
\item $\text{Put-$x$-above-some} = \abst{X}{\neg X,Z,m(x)\mminus}$,
\item $\text{Put-aside} = \abst{H}{\neg H,m(x)\pplus}$,
\item $\text{Put-above-$x$} = \abst{H}{\neg H,n(x)\pplus}$. \qed
\end{enumerate}

\Omit{
\alert{This is only correct if we allow unspecified decrements of the form 
$n:=n-\Delta$ for $0<\Delta\leq n$ in a way similar to increments.
We need this because putting $x$ above a tower, decreases $m(x)$ by the
number of blocks that end up below $x$. I vote for making this simple
generalization and making the whole framework symmetric with respect
to increments and decrements of counters.
Commented Theorem statement to make space.}
}


\section{Generalized Planning Revisited}

The notion of generalized planning can be extended to relational domains, where
instances share no common pool of actions, by moving to the abstract representation
provided by abstract actions:

\begin{definition}
  \label{def:gp:2}
  Let $\Q$ be a generalized planning problem, $F$ be a set of features, and $A_F$ be a set of sound abstract actions. 
  A policy for $\Q$ over $F$ is a partial mapping $\pi$ from the \emph{boolean valuations} over $F$ into $A_F$. 
  The (abstract) policy $\pi$ \emph{solves} $\Q$ if $\pi$ solves each instance $P$ in $\Q$; i.e.,
  if \emph{all} the trajectories $s_0,\ldots,s_{n}$ induced by $\pi$ in $P$ are goal reaching.
\end{definition}

Abstract  policies $\pi$ map boolean valuations over $F$ into sound abstract actions.
We write $a_i \in \pi(\phi_B(s_i))$ to express that  $a_i$ is one of the
concrete actions represented by the abstract action $\bar{a}_i=\pi(\phi_B(s_i))$
in the state $s_i$ of instance $P$. Since the abstract actions in $A_F$ are assumed to be sound,
there must be one such concrete action $a_i$ over any reachable state $s_i$ of any problem $P$  where $\bar{a}_i$ is applicable.
Such concrete action, however, is not necessarily unique.  The trajectories $s_0,\ldots,s_{n}$ induced by the policy $\pi$
in  $P$ are such that $a_i \in \pi(\phi_B(s_i))$, $s_{i+1}=f(a_i,s_i)$ for $i < n$, and $s_n$ is the
first state in the sequence  where the goal of $P$ is true,  $\pi(\phi_B(s_{n})) = \bot$, or $a_{n} \not\in A(s_{n})$. 

\Example
For the generalized problem $\Q_{clear}$ with features $F=\{H,n(x)\}$,
and actions \eqref{eq:ex1} and \eqref{eq:ex2}, the following policy,
expressed in compact form in terms of two rules, is a solution:
\begin{equation}
\Rule{\neg H, n(x)>0}{\bar{a}} \,,\quad \Rule{H, n(x)>0}{\bar{a}'}\,.
\label{eq:ex3}
\end{equation}
The policy picks blocks above $x$ and puts them aside
(not above $x$) until $n(x)$ becomes zero. \qed

\section{Computation}

We focus next on the computation of general policies. We proceed in two steps.
First, we map the generalized problem $\Q$ given features $F=\tup{B,N}$ and a set $A_F$ of
sound actions into a \emph{numerical non-deterministic} problem $Q_F$.
While numerical planning problems can be undecidable \cite{helmert:numerical},
these numerical problems are simpler and correspond to the so-called
\emph{qualitative numerical problems (QNPs)} whose solutions can be
obtained using standard, boolean FOND planners   \cite{srivastava:aaai2011,bonet:ijcai2017}.
\Omit{
For the second step, we map the QNP problem $Q_F$ into a standard (boolean)
FOND problem \cite{bonet:ijcai2017}.\footnote{The QNP to FOND translation of 
  \citeay{bonet:ijcai2017} is not always sound as claimed.
  We will present the fix elsewhere.
}
}

\subsection{Feature Projection $Q_F$}

For defining the first reduction, we assume a set $A_F$ of sound abstract actions
and formulas $I_F$ and $G_F$, defined over the atoms $p$ for $p\in B$ and $n=0$ for $n\in N$,  that provide a \emph{sound} approximation of  the initial and goal states of  the instances
in $\Q$. For this, it must be the case that for any instance $P$ in $\Q$:
a)~the truth valuation $\phi_B(s)$ for  the initial state $s$ of $P$ satisfies  $I_F$, and
b)~the reachable states $s$ in $P$ with a  truth valuation $\phi_B(s)$ that satisfies $G_F$ are goal states of $P$.

\begin{definition}
  \label{def:projection}
   For a given  set of sound abstract actions $A_F$, and sound  initial and goal formulas  $I_F$ and $G_F$ for $\Q$, 
   the \emph{projection} $Q_F=\tup{V_F,I_F,G_F,A_F}$ of $\Q$ is a  \emph{numerical, non-deterministic planning problem}
  with actions $A_F$,  initial and goal  situations $I_F$   and $G_F$, and state variables $V_F=F$.
\end{definition}

The states in a projection $Q_F$ are valuations $\bar{s}$ over the features $F$   that in $Q_F$, like in abstract actions,
represent \emph{state variables} and not functions. The possible initial states are the valuations
that satisfy $I_F$, the goal states are the ones that satisfy $G_F$, and the actions are the abstract actions in  $A_F$.

Solutions to a  projection $Q_F$ are partial policies $\pi$ that map \emph{boolean}
valuations over $F$ into actions in $A_F$ such that all the  state trajectories induced by $\pi$ in $Q_F$
are goal reaching.  A state trajectory $\bar{s}_0,\ldots,\bar{s}_{n}$ is induced by $\pi$ in $Q_F$
if $\bar{s}_0$ is a state that satisfies $I_F$, $\bar{a}_i=\pi(\bar{s}_i)$,  $\bar{s}_{i+1} \in \bar{F}(\bar{a}_i,\bar{s}_i)$
for $i < n$, and $\bar{s}_{n}$ is the first state in the sequence where $G_F$ is true, 
$\pi(\bar{s}_n)$ is not defined, or $\bar{a}_n$ is not applicable in $\bar{s}_n$.
$\bar{F}(\cdot,\cdot)$ is a non-deterministic transition function defined by the actions in $A_F$
in the usual  way.
\Omit{
Finally, $\bar{s}_{i+1}$ is a possible successor state of action $\bar{a}_i=\pair{Pre,\Eff}$
in state $\bar{s}_i$ where $Pre$ holds, i.e., $\bar{s}_{i+1} \in \bar{F}(\bar{a}_i,\bar{s}_i)$,
iff a)~the value of variables $p$ and $n$ not mentioned in $\Eff$ is the same in
$\bar{s}_{i}$ and $\bar{s}_{i+1}$,
b)~the value of $p$ in $\bar{s}_{i+1}$ is true (resp.\ false) iff $p$ (resp. $\neg p$) is
in $\Eff$, 
c)~the value of $n$ in $\bar{s}_{i+1}$ is higher than the value of $n$ in $\bar{s}_{i}$ if
$n\pplus$ is in $\Eff$, and
d)~the value of $n$ in $\bar{s}_{i+1}$ is the value of $n$ in $\bar{s}_{i}$ minus 1 if
$n\mminus$ is in $\Eff$.}
The soundness of $I_F$, $G_F$, and $A_F$ imply the soundness of the projection $Q_F$:

\begin{theorem}
  \label{thm:qf:soundness}
  If $A_F$ is a set of sound abstract actions for the features $F$, and the formulas $I_F$
  and $G_F$ are sound for $\Q$, then a solution $\pi$ for $Q_F$ is also a solution for $\Q$.
\end{theorem}

To see why this results holds, notice that if $\bar{a}_i$ is an abstract action applicable in $\bar{s}_i$ and $s$,
where $s$ is a state for some instance $P$ in $\Q$, then there is at least one concrete action $a_i$ in
$P$ that is represented by $\bar{a}_i$ in $s$. 
Also, for every state trajectory $s_0,\ldots,s_n$ induced by the policy $\pi$ on $P$, there
is one abstract state trajectory $\bar{s}_0,\ldots,\bar{s}_n$ induced by $\pi$ in $Q_F$ that \emph{tracks} the values of the features,
i.e.\ where $\bar{s}_i=\phi_F(s_i)$ for $i=0,\ldots,n$, as a result of the soundness of $A_F$ and the definition of $I_F$.
Since the latter trajectories are goal reaching, the former must be too since $G_F$ is sound.

The projections   $Q_F$ are  sound but not complete. The incompleteness is the result of the abstraction
(non-deterministic feature increments and decrements), and  the choice  of $A_F$,  $I_F$, and $G_F$  that
are  only assumed to be sound. 

\Example
Let us consider the features $F=\{H,n(x)\}$, the set of abstract actions $A_F=\{\bar{a},\bar{a}'\}$
given by 
\eqref{eq:ex1} and \eqref{eq:ex2}, $I_F=\{\neg H,n(x)>0\}$,
and $G_F=\{n(x) = 0\}$.
The policy given by  \eqref{eq:ex3} solves the projection $Q_F=\tup{V_F,I_F,G_F,A_F}$ of $\Q_{clear}$,
and hence, by Theorem~\ref{thm:qf:soundness},  also $Q_{clear}$. \qed

\subsection{Boolean Projection and FOND Problem $Q_F^+$}


The second piece of the computational model is the reduction of the projection $Q_F$ into
a standard (boolean) fully observable non-deterministic (FOND) problem.
For this we exploit a reduction from \emph{qualitative numerical planning} problems (QNPs)
into FOND  \cite{srivastava:aaai2011,bonet:ijcai2017}.
This reduction replaces the numerical variables $n$ by \emph{propositional symbols} named ``$n=0$''
that are meant to be true when the numerical variable $n$ has value zero.
The negation of the symbol ``$n=0$'' is denoted as ``$n > 0$''.

\begin{definition}
  Let $\Q$ be a generalized problem and  $Q_F$ be a  projection
  of $\Q$ for $F=\pair{B,N}$.   The \emph{boolean projection} $Q'_F$ associated with $Q_F$ is the FOND
  problem obtained from $Q_F$ by replacing
  1)~the numerical variables $n\in N$ by the symbols $n=0$,
  2)~first-order literals $n=0$ by propositional literals  $n=0$, 
  3)~effects $n\pplus$ by deterministic effects $n > 0$, and 
  4)~effects $n\mminus$ by non-deterministic effects $n>0\,|\,n=0$.
\end{definition}

The boolean projection $Q'_F$ is a FOND problem but neither the \emph{strong} or \emph{strong cyclic}
solutions of $Q'_F$ \cite{cimatti:strong-cyclic} capture the solutions of the numerical problem $Q_F$.
The reason is that the non-deterministic effects $n>0\,|\,n=0$ in $Q'_F$ are neither \emph{fair}, as
assumed in \emph{strong cyclic solutions}, nor \emph{adversarial}, as assumed in \emph{strong solutions}.
They are  \emph{conditionally fair}, meaning that from any time point on, infinite occurrences of
effects $n>0\,|\,n=0$ (decrements) imply the eventual outcome $n=0$, \emph{on the condition} that from
that point on, no action with effect $n > 0$ (increment) occurs.

The policies of the boolean FOND problem $Q'_F$ that capture the policies of
the numerical problem $Q_F$ are the strong cyclic solutions that \emph{terminate}
\cite{srivastava:aaai2011}. We call them the \emph{qualitative solutions} of
$Q'_F$ and define them equivalently as:


\begin{definition}
  A \emph{qualitative solution} of the FOND $Q'_F=\tup{V'_F,I'_F,G'_F,A'_F}$ is a partial mapping $\pi$
  from \emph{boolean feature valuations} into actions in $A'_F$ such that the state-action trajectories
  induced by $\pi$ over $Q'_F$ \emph{that are conditionally fair} are all goal reaching.
\end{definition}

A state-action trajectory  $s'_0, a'_0, s'_1,\ldots$ over $Q'_F$ is \emph{not} conditionally fair
iff a)~it is infinite, b)~after a certain time step $i$, it contains infinite actions with effects
$n>0\,|\,n=0$ and no action with effect $n>0$, and c)~there is no time step after $i$ where $n=0$.
The qualitative solutions of $Q'_F$ capture the solutions of the numerical projection $Q_F$
exactly:

\begin{theorem}
  $\pi$ is a qualitative solution of the boolean FOND $Q'_F$ iff $\pi$ is a solution of the numerical projection $Q_F$.
\end{theorem}

This is because for every trajectory $s_0,s_1,\ldots$ induced by policy $\pi$ over the numerical problem
$Q_F$, there is a trajectory $s'_0,s'_1,\ldots$ induced by $\pi$ over the boolean problem $Q'_F$, and vice
versa, where $s'_i$ is the boolean projection of the state $s_i$; namely, $p$ is true in $s'_i$ iff $p$ is
true in $s_i$, and $n=0$ is true (resp.\ false) in $s'_i$ iff $n$ has value (resp.\ greater than) $0$ in $s_i$. 


We say that $Q^+_F$ is a \emph{qualitative FOND} associated with the boolean projection $Q'_F$
if the strong cyclic solutions $\pi$ of $Q^+_F$ represent \emph{qualitative solutions} of $Q'_F$;
i.e., strong cyclic solutions of $Q'_F$ that \emph{terminate}  \cite{srivastava:aaai2011}.
Under some conditions, roughly, that variables that are decreased by some action are not increased
by other actions, $Q^+_F$ can be set to $Q'_F$ itself. In other cases, a suitable translation is
needed  \cite{bonet:ijcai2017}. Provided a sound  translation,\footnote{
The translation by \citeay{bonet:ijcai2017} is actually not sound in general
as claimed. We'll report the fix elsewhere.}   solutions to a generalized problem $\Q$ 
can be computed from  $Q^+_F$ using \emph{off-the-shelf FOND planners}:\footnote{For
this, the formulas $I_F$ and $G_F$ must be expressed in DNF. Compiling them into
conjunctions of literals, as expected by FOND planners, is direct using extra
atoms.}

\begin{theorem}
  \label{thm:main}
  Let $Q_F$ be a feature projection of a generalized problem $\Q$, and let $Q^+_F$ be a 
  qualitative FOND problem associated with the boolean projection $Q'_F$.
  The strong cyclic solutions of $Q^+_F$ are solutions of the generalized problem $\Q$.
\end{theorem}

This is a soundness result. Completeness is lost already  in the reduction from $\Q$ to $Q_F$
as discussed above.


\Omit{
Completeness is not possible at this stage since, as we have seen above,
the feature projection $Q_F$ itself is not complete. 
The qualitative FOND problem $Q^+_F$ differs from $Q'_F$ in the addition fluents and
actions that ensure that non-deterministic effects $n>0\,|\,n=0$ have fair behavior in loops 
where there is no action with effect $n>0$.
}

\Example
For $\Q_{clear}$ and the projection $Q_F$ above with $F=\{H,n(x)\}$, the boolean projection
$Q'_F=\tup{V'_F,I'_F,G'_F,A'_F}$ has boolean variables $V'_F=\{H,n(x)=0\}$, initial and goal
formulas $I'_F=\{\neg H, n(x) > 0\}$ and $G'_F=\{n(x)=0\}$, and actions $A'_F=\{\bar{a}'_1,\bar{a}'_2\}$
where $\bar{a}'_1=\abst{\neg H,n(x)>0}{H,n(x) > 0\,|\,n(x)=0}$ and $\bar{a}'_2=\abst{H}{\neg H}$.
Since there are no actions in $Q_F$ that increment the numerical variable $n(x)$, 
$Q^+_F$ is $Q'_F$ and hence, by Theorem~\ref{thm:main},
the strong cyclic solutions to $Q'_F$ are
solutions to $\Q_{clear}$. The policy shown in \eqref{eq:ex3} was computed from $Q'_F$
by the FOND planner MyND \cite{mynd} in 58 milliseconds. \qed

\Omit{\tiny
\begin{verbatim}
Outout for Qclear1: complete set of actions
Preprocess: 0.000, Search total: 0.048
if { (not (h)), (not (zero nx)) } then pick-above-x;
if { (h), (not (zero nx)) } then put-aside;

Output for Qclear2: 2 features/2 actions, p02: n(x) and m(x) > 0
Preprocess: 0.000, Search total: 0.058
if { (not (h)), (not (q mx)), (not (q nx)), (not (x)), (not (zero mx)), (not (zero nx)) } then set nx;
if { (not (h)), (not (q mx)), (q nx), (not (x)), (not (zero mx)), (not (zero nx)) } then pick-above-x;
if { (h), (not (q mx)), (q nx), (not (x)), (not (zero mx)), (not (zero nx)) } then put-aside;
\end{verbatim}
}

\section{Examples and Experiments}

We illustrate the  representation changes and the resulting methods for computing policies
in four problems. Experiments were done on an Intel i5-4670 CPU with 8Gb of RAM.

\subsection{Moving in Rectangular Grids}

The generalized problem $\Q_{move}$ involves 
an agent that moves in an arbitrary $n\times m$ grid. 
The instances $P$ are represented with atoms $at(x,y)$ and 
actions $\Move(x,y,x',y')$, and have goals of the form $at(x^*,y^*)$.
A general policy for $\Q_{move}$ can be obtained
by introducing the features 
$\Delta_X=|x^*-x_s|$ and $\Delta_Y=|y^*-y_s|$
where $at(x_s,y_s)$ is true in the state $s$.
A projection $Q_F$ of $\Q_{move}$ is obtained
from this set of features $F$, the goal formula
$G_F=\{\Delta_X=0,\Delta_Y=0\}$, and the initial DNF formula 
$I_F$ with four terms corresponding to the four 
truth valuations of the atoms $\Delta_X=0$ and $\Delta_Y=0$.
The set $A_F$ of abstract actions
\begin{enumerate}[--]
\item $\text{Move-in-row} =  \abst{\Delta_X>0}{\Delta_X\mminus}$,
\item $\text{Move-in-column} = \abst{\Delta_Y>0}{\Delta_Y\mminus}$
\end{enumerate}
is sound and captures the actions that move the agent
toward the target horizontally and  vertically.
$A_F$ is not complete as it is missing the  two abstract actions
for moving away from the target. The boolean projection $Q'_F$ is $Q_F$
with the atoms $\Delta_X=0$ and $\Delta_Y=0$, and their negations, interpreted
as propositional literals, and the two actions transformed as:
\begin{enumerate}[--]
\item $\text{Move-in-row}' = \abst{\Delta_X>0}{\Delta_X=0\,|\,\Delta_X>0}$,
\item $\text{Move-in-column}'=  \abst{\Delta_Y>0}{\Delta_Y=0\,|\,\Delta_Y>0}$.
\end{enumerate}
The resulting FOND problem $Q^+_F$ is equal to $Q'_F$.
The  MyND planner yields the policy below in 54 milliseconds, which from
Theorem~\ref{thm:main}, is a solution  of $\Q_{move}$:
\begin{enumerate}[--]
\item \Rule{\Delta_X>0, \Delta_Y>0}{\text{Move-in-row}'},
\item \Rule{\Delta_X=0, \Delta_Y>0}{\text{Move-in-column}'}.
\end{enumerate}

\Omit{\tiny
\begin{verbatim}
Output for Qmove:
Preprocess: 0.000, Search total: 0.054
if { (not (zero dx)), (not (zero dy)) } then move-horiz;
if { (zero dx), (not (zero dy)) } then move-vert;
\end{verbatim}
}

\subsection{Sliding Puzzles}

We consider  next  the generalized problem $\Q_{slide}$ where 
a designated tile $t^*$ must be moved to a target location
$(x^*_t,y^*_t)$ in a sliding puzzle.
The STRIPS encoding of an instance $P$
contains atoms $at(t,x,y)$ and $atB(x,y)$
for the location of tiles and  the ``blank'', 
and actions  $\Move(t,x,y,x',y')$ for exchanging the
location of tile $t$ and the blank if in adjacent cells. 
For solving $\Q_{slide}$, we consider two numerical features:
the Manhattan distance $\Delta_t$ from the current location
of the target tile to its target location,
and the minimal total distance $\Delta_b$ that the blank
must traverse without going through the current target location
so that that target tile can be moved and decrement the value
of $\Delta_t$.

The feature projection $Q_F$ has goal formula $G_F=\{\Delta_t=0\}$,
initial formula $I_F$ given by the DNF with four terms corresponding
to the four truth valuations of the atoms $\Delta_t=0$
and $\Delta_b=0$, and the two abstract actions
\begin{enumerate}[--]
\item $\text{Move-blank} =  \abst{\Delta_b>0}{\Delta_b\mminus}$,
\item $\text{Move-tile}= \abst{\Delta_b=0,\Delta_t>0}{\Delta_t\mminus,\Delta_b\pplus}$.
\end{enumerate}
The boolean projection $Q'_F$ is $Q_F$ with the atoms $\Delta_t=0$ and $\Delta_b=0$,
and their negations, interpreted as propositional literals, and the two actions 
transformed as:
\begin{enumerate}[--]
\item $\text{Move-blank}' : \abst{\Delta_b>0}{\Delta_b>0\,|\,\Delta_b=0}$,
\item $\text{Move-tile}' :  \abst{\Delta_b=0,\Delta_t>0}{\Delta_b>0, \Delta_t>0\,|\,\Delta_t=0}$.
\end{enumerate}
As before, $Q^+_F$ is equal to $Q'_F$ because the only action that
increments a variable, $\Delta_b$, cannot be used until $\Delta_b=0$.
MyND solves $Q^+_F$ in 65 milliseconds, producing the policy below which by Theorem~\ref{thm:main}
solves the generalized problem $\Q_{move}$:
\begin{enumerate}[--]
\item \Rule{\Delta_b>0, \Delta_t>0}{\text{Move-blank}'},
\item \Rule{\Delta_b=0, \Delta_t>0}{\text{Move-tile}'}.
\end{enumerate}

\Omit{\tiny
\begin{verbatim}
Output for Qslide:
Preprocess: 0.000, Search total: 0.065
if { (not (zero delta-b)), (not (zero delta-t)) } then move-blank;
if { (zero delta-b), (not (zero delta-t)) } then move-tile;
\end{verbatim}
}

\subsection{Blocksworld: Achieving $on(x,y)$}

The problem $\Q_{on}$ is about achieving goals of the form $on(x,y)$ in Blocksworld instances where
for simplicity, the gripper is initially empty, and the blocks $x$ and $y$ are in different towers with
blocks above them. We use a set of features $F$ given by $n(x)$ and $n(y)$ for the number of blocks
above $x$ and $y$,  booleans $X$ and $H$ that are true when the gripper is holding $x$ or another block,
and $on(x,y)$ that is true when $x$ is on $y$.
\Omit{
\begin{enumerate}[--]
\item $n(x)$ and $n(y)$: $\#$ of blocks above $x$ and above $y$,
\item $X, H$: holding block $x$ and another block, 
\item $on(x,y)$ $x$ is on $y$
\end{enumerate}
}
%
As before, 
we include a sound but incomplete set of actions
$A_F$ needed to solve $\Q_{on}$ where $E$ abbreviates the conjunction $\neg X$ and
$\neg H$:

\begin{enumerate}[--]
\item $\text{Pick-$x$} = \abst{E, n(x)=0}{X}$,
\item $\text{Pick-above-$x$} = \abst{E, n(x) > 0}{H, n(x)\mminus}$,
\item $\text{Pick-above-$y$} = \abst{E,  n(y) > 0}{H, n(y)\mminus}$,
\item $\text{Put-$x$-on-$y$} = \abst{X, n(y) = 0}{\neg X, on(x,y), n(y)\pplus}$,
\item $\text{Put-other-aside} = \abst{H}{\neg H}$.
\end{enumerate}

\noindent The projected problem $Q_F$ has this set of actions $A_F$, $I_F=\{n(x) > 0, n(y) > 0, E, \neg on(x,y)\}$,
and $G_F=\{on(x,y)\}$.
The projection $Q'_F$ is $Q_F$ but with the propositional reading of the atoms $n(\cdot)=0$ and
their negations, and actions $A'_F$:

\begin{enumerate}[--]
\item $\text{Pick-$x$}' = \abst{E, n(x)=0}{X}$,
\item $\text{Pick-above-$x$}' = \abst{E, n(x)=0}{H, n(x) > 0 \, | \, n(x)=0}$,
\item $\text{Pick-above-$y$}' = \abst{E, n(y)=0}{H, n(y) > 0 \, | \, n(y)=0}$, 
\item $\text{Put-$x$-on-$y$}' = \abst{X, n(y)>0}{\neg X, on(x,y), n(y) > 0}$
\item $\text{Put-other-aside}' = \abst{H}{\neg H}$.
\end{enumerate}

\noindent Since the effects that increment a variable also achieve the goal, the qualitative
problem $Q^+_F$ is equal to $Q'_F$. The planner MyND over $Q^+_F$ yields  the policy $\pi$
below in 70 milliseconds, that solves $Q^+_F$ and hence $Q_{on}$.
The negated goal condition $\neg on(x,y)$ is part of the following rules but it is
for reasons of clarity:
\begin{enumerate}[--]
\item \Rule{E, n(x) > 0, n(y) > 0}{\text{Pick-above-$x$}'},
\item \Rule{H, \neg X, n(x) > 0, n(y) > 0}{\text{Put-other-aside}'},
\item \Rule{H, \neg X, n(x) = 0, n(y) > 0}{\text{Put-other-aside}'},
\item \Rule{E, n(x) = 0, n(y) > 0}{\text{Pick-above-$y$}'},
\item \Rule{H, \neg X, n(x) = 0, n(y) = 0}{\text{Put-other-aside}'},
\item \Rule{E, n(x) = 0, n(y) = 0}{\text{Pick-above-$x$}'},
\item \Rule{X, \neg H, n(x) = 0, n(y) = 0}{\text{Put-$x$-on-$y$}'}.
\end{enumerate}

\subsection{Blocksworld: Building a Tower}

We consider a final  generalized blocks problem, $Q_{tower}$, where the task is building a tower with all the blocks.
For this,  we consider the feature set $F'$ and the set of abstract actions $A_{F'}$ above (end of Sect.~3),
with  $I_{F'}=\{\neg X,\neg H,Z,n(x)>0,m(x)>0\}$ and $G_{F'}=\{\neg X, \neg H, m(x)=0\}$.
\Omit{
the gripper is holding the block $x$ or another block, $Z$ tells whether there is some block
below $x$, $n(x)$ counts the number of blocks above $x$, and $m(x)$ counts the number of blocks
that are different from $x$ and the block being held, and also different for the blocks that
are in the same tower that $x$ is (if $x$ is not being held).
Since $A_{clear}$ is a complete set, such abstract actions are enough to define a general
plan for $Q_{tower}$.}
For space reasons we do not show the  projections $Q_{F'}$ and $Q'_{F'}$, or the 
problem  $Q^+_{F'}$ fed to the planner. The resulting policy 
\begin{enumerate}[--]
\item \Rule{\neg X, \neg H, m(x)>0}{\text{Pick-other}},
\item \Rule{\neg X, H, m(x)>0}{\text{Put-above-$x$}},
\item \Rule{\neg X, H, m(x)=0}{\text{Put-above-$x$}}.
\end{enumerate}
is obtained  with  the FOND-SAT-based planner \cite{tomas:fond-sat} in 284 milliseconds
(the FOND planner MyND produces a buggy plan in this case).
Interestingly, the addition of the atom $\neg Z$ to $G_F$ yields a very different policy
that builds a single tower but with $x$ at the bottom.

\section{Discussion}



\noindent\textbf{Arbitrary Goals, Concepts, Indexicals, and Memory.}
Most of the  examples above   deal with instances  involving atomic goals
(except $\Q_{tower}$).  The definition of a generalized problem that would yield a policy for solving \emph{any Blocksworld instance}
is more challenging. Inductive, as opposed to model-based approaches, for obtaining such policies have been reported
\cite{martin-geffner:generalized,fern:generalized}. These approaches learn generalized policies from sampled instances and
their plans. They do not learn boolean and numerical features but unary predicates or concepts like ``the clear block that
is above $x$'', yet features,  concepts, and   indexical or deictic representations \cite{chapman:pointers,ballard:pointers}
are closely related to each other. The execution of general policies requires tracking a fixed number of features,
and hence \emph{constant memory}, independent of the instance size, that is given by the number of features.
This is relevant from a cognitive point view where short term memory is bounded and small \cite{ballard:memory}.



\medskip\noindent\textbf{General Policy Size, Polynomial Features, and Width.}
The numerical  feature $h^*$ that measures the optimal distance  to the  goal can be used in the policy $\pi_{Gen}$ given by the 
rule $\Rule{h^* > 0}{MoveToGoal}$ for solving \emph{any problem}. Here $MoveToGoal$ is the abstract action $\bar{a}=\abst{h^* > 0}{h^*\mminus}$
that is sound and represents any concrete optimal  action. The problem with the feature $h^*$  is that its computation is intractable in general,
thus  it is reasonable to impose the requirement that \emph{features should be computable in (low) polynomial time.}
Interestingly, however, instances over many of the standard classical  domains  featuring   \emph{atomic goals}
have a \emph{bounded and small width}, which implies that they can be solved optimally in low polynomial time \cite{nir:ecai2012}.
This means that the general policy $\pi_{Gen}$  is not useless after all, as it solves all such instances in polynomial time.
The general policy $\pi$ computed by FOND planners above, however, involves  a fixed set of boolean variables that  can be tracked
in time that  is \emph{constant} and does not depend on the instance size.



\medskip\noindent\textbf{Deep Learning of Generalized Policies, and Challenge.}
Deep learning methods have been recently used for learning
generalized policies  in   domains like Sokoban \cite{drl:sokoban} and 3D navigation \cite{drl:3dnav}. 
Deep learning methods however learn functions from \emph{inputs of a fixed size}. Extensions  for dealing with images
or strings of arbitrary size have been developed  but images and strings have a  simple 2D or linear structure. 
The structure of relational representations is not so uniform and this is probably one of the main reasons
that we have not yet seen deep (reinforcement) learning methods
being used to solve \emph{arbitrary instances} of the blocks world.  Deep learning methods are good for learning features
but in order to be applicable in such domains, they need the inputs expressed in terms of a fixed number
of features as well,  such as those  considered in this work. 
The open challenge is to  learn such features from data. 
The relevance  of this work to learning is that it makes precise \emph{what} needs to  be learned.
A crisp requirement is that  the set of features $F$ to be learned for a generalized problem  $\Q$
should support \emph{a sound set of actions} $A_F$ sufficient for solving the problem.

\section{Summary}

We have extended the standard semantic formulation of generalized planning to domains with instances
that do not have actions in common by introducing abstract actions: actions that operate on the common
pool of features, and whose soundness and completeness can be determined.
General plans map features into abstract actions, which if sound, can always be instantiated with a
concrete action. By a series of reductions, we have also shown how to obtain such policies using
off-the-shelf FOND planners. The work relates to a number of concepts and threads in AI, and raises a
number of crisp challenges, including the automatic discovery of the boolean and integer features that
support general plans in relational domains. 

\section*{Acknowledgements}

We thank the anonymous reviewers for useful comments.
H. Geffner is partially funded by grant TIN-2015-67959-P, MINECO, Spain.

\bibliographystyle{named}
\bibliography{control}

\begin{thebibliography}{}

\bibitem[\protect\citeauthoryear{Ballard \bgroup \em et al.\egroup
  }{1995}]{ballard:memory}
Dana~H. Ballard, Mary~M. Hayhoe, and Jeff~B. Pelz.
\newblock Memory representations in natural tasks.
\newblock {\em Journal of Cognitive Neuroscience}, 7(1):66--80, 1995.

\bibitem[\protect\citeauthoryear{Ballard \bgroup \em et al.\egroup
  }{1997}]{ballard:pointers}
Dana~H. Ballard, Mary~M. Hayhoe, Polly~K. Pook, and Rajesh P.~N. Rao.
\newblock {Deictic codes for the embodiment of cognition}.
\newblock {\em Behavioral and Brain Sciences}, 20(4):723--742, 1997.

\bibitem[\protect\citeauthoryear{Belle and Levesque}{2016}]{BelleL16}
Vaishak Belle and Hector~J. Levesque.
\newblock Foundations for generalized planning in unbounded stochastic domains.
\newblock In {\em {KR}}, pages 380--389, 2016.

\bibitem[\protect\citeauthoryear{Bonet \bgroup \em et al.\egroup
  }{2009}]{bonet09automatic}
Blai Bonet, Hector Palacios, and Hector Geffner.
\newblock Automatic derivation of memoryless policies and finite-state
  controllers using classical planners.
\newblock In {\em ICAPS}, pages 34--41, 2009.

\bibitem[\protect\citeauthoryear{Bonet \bgroup \em et al.\egroup
  }{2017}]{bonet:ijcai2017}
Blai Bonet, Giuseppe {De Giacomo}, Hector Geffner, and Sasha Rubin.
\newblock Generalized planning: Non-deterministic abstractions and trajectory
  constraints.
\newblock In {\em Proc. IJCAI}, pages 873--879, 2017.

\bibitem[\protect\citeauthoryear{Boutilier \bgroup \em et al.\egroup
  }{2001}]{boutilier2001symbolic}
Craig Boutilier, Ray Reiter, and Bob Price.
\newblock Symbolic dynamic programming for first-order {MDPs}.
\newblock In {\em IJCAI}, volume~1, pages 690--700, 2001.

\bibitem[\protect\citeauthoryear{Chapman}{1989}]{chapman:pointers}
David Chapman.
\newblock {Penguins can make cake}.
\newblock {\em AI magazine}, 10(4):45--50, 1989.

\bibitem[\protect\citeauthoryear{Cimatti \bgroup \em et al.\egroup
  }{2003}]{cimatti:strong-cyclic}
Alessandro Cimatti, Marco Pistore, Marco Roveri, and Paolo Traverso.
\newblock Weak, strong, and strong cyclic planning via symbolic model checking.
\newblock {\em Artificial Intelligence}, 147(1-2):35--84, 2003.

\bibitem[\protect\citeauthoryear{Fern \bgroup \em et al.\egroup
  }{2006}]{fern:generalized}
Alan Fern, Sung~Wook Yoon, and Robert Givan.
\newblock Approximate policy iteration with a policy language bias: Solving
  relational markov decision processes.
\newblock {\em JAIR}, 25:75--118, 2006.

\bibitem[\protect\citeauthoryear{Geffner and Geffner}{2018}]{tomas:fond-sat}
Tomas Geffner and Hector Geffner.
\newblock Compact policies for non-deterministic fully observable planning as
  {SAT}.
\newblock In {\em Proc. ICAPS}, 2018.
\newblock To appear.

\bibitem[\protect\citeauthoryear{Groshev \bgroup \em et al.\egroup
  }{2017}]{drl:sokoban}
Edward Groshev, Maxwell Goldstein, Aviv Tamar, Siddharth Srivastava, and Pieter
  Abbeel.
\newblock Learning generalized reactive policies using deep neural networks.
\newblock {\em arXiv preprint arXiv:1708.07280}, 2017.

\bibitem[\protect\citeauthoryear{Helmert}{2002}]{helmert:numerical}
Malte Helmert.
\newblock Decidability and undecidability results for planning with numerical
  state variables.
\newblock In {\em Proc. ICAPS}, pages 44--53, 2002.

\bibitem[\protect\citeauthoryear{Hu and {De Giacomo}}{2011}]{hu:generalized}
Yuxiao Hu and Giuseppe {De Giacomo}.
\newblock Generalized planning: Synthesizing plans that work for multiple
  environments.
\newblock In {\em IJCAI}, pages 918--923, 2011.

\bibitem[\protect\citeauthoryear{Lipovetzky and Geffner}{2012}]{nir:ecai2012}
Nir Lipovetzky and Hector Geffner.
\newblock Width and serialization of classical planning problems.
\newblock In {\em Proc. ECAI}, pages 540--545, 2012.

\bibitem[\protect\citeauthoryear{Mart{\'\i}n and
  Geffner}{2004}]{martin-geffner:generalized}
Mario Mart{\'\i}n and Hector Geffner.
\newblock Learning generalized policies from planning examples using concept
  languages.
\newblock {\em Applied Intelligence}, 20(1):9--19, 2004.

\bibitem[\protect\citeauthoryear{Mattm{\"u}ller \bgroup \em et al.\egroup
  }{2010}]{mynd}
Robert Mattm{\"u}ller, Manuela Ortlieb, Malte Helmert, and Pascal Bercher.
\newblock Pattern database heuristics for fully observable nondeterministic
  planning.
\newblock In {\em Proc. ICAPS}, pages 105--112, 2010.

\bibitem[\protect\citeauthoryear{Mirowski \bgroup \em et al.\egroup
  }{2016}]{drl:3dnav}
Piotr Mirowski, Razvan Pascanu, Fabio Viola, Hubert Soyer, Andrew~J. Ballard,
  Andrea Banino, Misha Denil, Ross Goroshin, Laurent Sifre, Koray Kavukcuoglu,
  Dharshan Kumaran, and Raia Hadsell.
\newblock Learning to navigate in complex environments.
\newblock {\em arXiv preprint arXiv:1611.03673}, 2016.

\bibitem[\protect\citeauthoryear{Segovia \bgroup \em et al.\egroup
  }{2016}]{anders:generalized}
Javier Segovia, Sergio Jim{\'e}nez, and Anders Jonsson.
\newblock Generalized planning with procedural domain control knowledge.
\newblock In {\em Proc. ICAPS}, pages 285--293, 2016.

\bibitem[\protect\citeauthoryear{Srivastava \bgroup \em et al.\egroup
  }{2008}]{srivastava08learning}
Siddharth Srivastava, Neil Immerman, and Shlomo Zilberstein.
\newblock Learning generalized plans using abstract counting.
\newblock In {\em AAAI}, pages 991--997, 2008.

\bibitem[\protect\citeauthoryear{Srivastava \bgroup \em et al.\egroup
  }{2011a}]{srivastava:generalized}
Siddharth Srivastava, Neil Immerman, and Shlomo Zilberstein.
\newblock A new representation and associated algorithms for generalized
  planning.
\newblock {\em Artificial Intelligence}, 175(2):615--647, 2011.

\bibitem[\protect\citeauthoryear{Srivastava \bgroup \em et al.\egroup
  }{2011b}]{srivastava:aaai2011}
Siddharth Srivastava, Shlomo Zilberstein, Neil Immerman, and Hector Geffner.
\newblock Qualitative numeric planning.
\newblock In {\em AAAI}, pages 1010--1016, 2011.

\bibitem[\protect\citeauthoryear{{van Otterlo}}{2012}]{van2012solving}
Martijn {van Otterlo}.
\newblock Solving relational and first-order logical markov decision processes:
  A survey.
\newblock In {\em Reinforcement Learning}, pages 253--292. Springer, 2012.

\bibitem[\protect\citeauthoryear{Wang \bgroup \em et al.\egroup
  }{2008}]{wang2008first}
Chenggang Wang, Saket Joshi, and Roni Khardon.
\newblock First order decision diagrams for relational {MDPs}.
\newblock {\em Journal of Artificial Intelligence Research}, 31:431--472, 2008.

\end{thebibliography}

\end{document}